\documentclass{article}
\usepackage{ijcai16}

\usepackage{times}

\usepackage{amsmath,bm}

\usepackage{algorithm}
\usepackage{algorithmic}
\usepackage[english]{babel}

\usepackage{enumerate}
\usepackage{amsfonts}
\usepackage{subfigure}
\usepackage{graphicx}
\usepackage{amsthm}
\usepackage{epsfig}

\newcommand{\Xv}{\mathbf{X}}

\newcommand{\argmin}{\operatornamewithlimits{argmin}}

\usepackage{color}

\newtheorem{theorem}{Theorem}

\title{A Communication-Efficient Parallel Method for Group-Lasso}
\author{Binghong Chen, Jun Zhu \\
Tsinghua University, Beijing  \\
cbh13@mails.tsinghua.edu.cn, dcszj@mail.tsinghua.edu.cn}

\begin{document}

\maketitle

\begin{abstract}
Group-Lasso (gLasso) identifies important explanatory factors in predicting the response variable by considering the grouping structure over input variables. However, most existing algorithms for gLasso are not scalable to deal with large-scale datasets, which are becoming a norm in many applications.
In this paper, we present a divide-and-conquer based parallel algorithm (DC-gLasso) to scale up gLasso in the tasks of regression with grouping structures. DC-gLasso only needs two iterations to collect and aggregate the local estimates on subsets of the data, and is provably correct to recover the true model under certain conditions. We further extend it to deal with overlappings between groups. Empirical results on a wide range of synthetic and real-world datasets show that DC-gLasso can significantly improve the time efficiency without sacrificing regression accuracy.
\end{abstract}

\section{Introduction}

In many regression problems, we are interested in identifying important explanatory factors in predicting the response variable. Lasso~\cite{tibshirani1996regression} represents a type of widely applied methods with sound theoretical guarantee. To consider the settings where each explanatory factor may be represented by a group of derived input variables, group-Lasso (gLasso) has been developed~\shortcite{yuan2006model}, which yields group-wise sparse estimates. gLasso has been applied in various applications, including multifactor analysis-of-variance (ANOVA) problem~\cite{yuan2006model}, learning pairwise interactions between regression factors~\cite{lim2013learning}, solving EEG source problems arising from visual activation studies~\cite{lim2013group}, estimating breeding values using molecular markers spread over the whole genome~\cite{ogutu2014regularized}, visual saliency detection~\cite{souly2015visual} and functional MRI data analysis ~\cite{shimizu2015toward}.

Many algorithms have been developed to solve the optimization problem of gLasso, including blockwise coordinate descent (BCD)~\cite{yuan2006model,meier2008group}, (fast) iterative shrinkage-thresholding algorithm (ISTA/FISTA)
~\cite{beck2009fast,liu2009slep,villa2014proximal},  hybrid BCD and ISTA~\cite{qin2013efficient}, alternating direction method (ADM)~\cite{qin2012structured} and groupwise-majorization-descent (GMD)~\cite{yang2014fast}. Among these algorithms, GMD was reported to run faster than BCD and FISTA~\cite{yang2014fast} on a single machine.
However, with the fast-growing volume of datasets, the storage and computation of data in one single machine become difficult, and the need to design a scalable algorithm for regression with grouping structure is urgent.
Previous work has been done to parallelize computation on the level of one iteration in an iterative procedure~\cite{peng2013parallel}. Although this type of parallel computing can save computation time, the time reduced is often limited due to the frequent communication between machines.

In contrast, a divide-and-conquer (DC) approach performs parallel computing at the level of the whole optimization process, where the dataset is split into shards and each worker handles a subset of data locally to produce local estimates. Then, a master node collects the local estimates and combines them to obtain a final estimate according to some aggregation strategy. For linear models, {\it averaging} is the simplest and most popular strategy, which was proposed by Mcdonald et al.~\shortcite{mcdonald2009efficient} and carefully studied by Zinkevich et al.~\shortcite{zinkevich2010parallelized} as well as Zhang, Duchi and Wainwright~\shortcite{zhang2012communication}.
Wang et al.~\shortcite{wang2014median} adopt a similar idea and develop a two-step parallel inference algorithm for Lasso, where in the model selection phase, the variables are selected using majority voting from the local Lasso estimates; and in the coefficient estimation phase, the coefficients of the selected variables are estimated by averaging the local ordinary least square (OLS) estimates.

In this paper, we adopt the DC framework and present a parallel inference algorithm for gLasso (DC-gLasso) in the regression tasks with grouping structures among input variables. Similar to the algorithm proposed by Wang et al.~\shortcite{wang2014median}, our algorithm has two distributed computing steps --- it aggregates the local sparse estimates by gLasso to select a subset of variables via majority voting at the first step and then averages the local coefficients estimated by OLS regression on the selected variables to get the final estimate. We show that our algorithm successfully scales up regression with grouped variables by theoretical analysis and experiment evidence. We further extend the method to deal with overlapping structures among feature groups~\cite{Jacob2009Group}.

The rest of the paper is organized as follows. We will formulate the problem and introduce group-Lasso in next Section. In Section 3, we will present DC-gLasso in detail. Finally, we will present some theoretical analysis (Section 4) and empirical results (Section 5).



\section{Problem formulation and group-Lasso}

Let $\mathbf{X}$ denote a given $n\times p$ design matrix, where $n$ is the number of observations and $p$ is the number of features. We consider the linear regression model, where the response variables $Y \in \mathbb{R}^n$ are modeled as:
\begin{equation}
\label{eq1}
Y=\mathbf{X}\beta+\epsilon,
\end{equation}
$\beta$ is a $p$-dimensional coefficient vector, and $\epsilon$ is the Gaussian white noise vector with mean $0$ and variance $\sigma^2$.

An optimal weight vector $\beta$ can be learned by minimizing a squared error with some regularization, e.g., $L_2$-norm in ridge regression. In many scenarios, we expect to have a sparse model (i.e., most of the elements of $\beta$ are zero), which is useful to identify a set of important explanatory factors for predicting the response variables and meanwhile protect the model from over-fitting. Various techniques have been developed for deriving a sparse estimate, with Lasso~\cite{tibshirani1996regression} as one of the most extensively studied examples. Lasso performs an $L_1$-norm regularized linear regression problem and selects each individual elements.

In many applications, the variables are not independent. Instead, they may have some grouping structure. We would expect to select the variables at the group level. That is, a group of correlated variables is selected or discarded at the same time. Formally, let $q$ denote the number of groups, and $d_i$ denote the size of group $i$. So $d_i$ must satisfy the condition that $\sum_{i=1}^{q}d_i=p$. Accordingly, the design matrix can be partitioned into $q$ sub-matrices $\Xv_i \in \mathbb{R}^{ n \times d_i}$, and the linear regression model can be restated as:
\begin{equation}
\label{eq2}
Y=\sum_{i=1}^{q} \mathbf{X}_i\beta_i + \epsilon,
\end{equation}
where $\beta_i$ is the corresponding $d_i$-dimensional coefficient vector for the variables in group $i$. The groupwise sparsity means that the groupwise vector $\beta_i$ is zero or not.
To obtain the above sparsity, group-Lasso (gLasso)~\cite{yuan2006model} solves the $L_{2,1}$-norm regularized linear regression problem:
\begin{equation}\label{eq:gLasso}
\hat{\beta} = \argmin_{\beta}  \left\| Y - \mathbf{X}\beta \right\|_2^2 + \lambda \sum_{i=1}^{q} \left\| \beta_i \right\|_2,
\end{equation}
where $\lambda$ is a positive regularization parameter to control the sparsity level of the estimates. 

The gLasso problem (\ref{eq:gLasso}) is convex. Several solvers have been developed, including the classical Block Coordinate Descent (BCD) method~\cite{yuan2006model,meier2008group}, the gradient based ISTA/FISTA method and its extension~\cite{beck2009fast,liu2009slep,villa2014proximal}, the extension and combination of BCD and ISTA~\cite{qin2013efficient}, the alternating direction method (ADM)~\cite{qin2012structured} and the groupwise-majorization-descent (GMD) algorithm~\cite{yang2014fast}. Among these algotithms, GMD was reported to run much faster than BCD and FISTA~\cite{yang2014fast}.

However, all the above methods focus on solving the gLasso problem on a single machine. With the increase in the size of available data and the amount of exploitable computation resources, it is desirable to design a communication-efficient distributed algorithm to solve this optimization problem. In next section, we present a low-communication-cost parallel algorithm to solve the group-structured regression problem. As we shall see, our method is compatible with all the above single-machine algorithms to solve the sub-gLasso problems on each local machine.

\vspace{-.1cm}
\section{A Divide-and-Conquer Algorithm}
\vspace{-.1cm}

We now present a divide-and-conquer based parallel algorithm for group-Lasso. We build our work on the recent progress of MEdian Selection Subset AGgregation Estimator (MESSAGE)
~\cite{wang2014median} for solving Lasso problems on large-scale data sets.
MESSAGE is a divide-and-conquer algorithm that parallelizes Lasso selection, and it has excellent performance in variable selection, estimation, prediction, and computation time compared to alternative methods. 


We adopt a similar divide-and-conquer approach in the regression tasks with grouping structure among variables. As outlined in Algorithm~\ref{alg-DC-gLasso}, our DC-gLasso algorithm consists of two stages --- the {\it model selection} stage and the {\it coefficient estimation} stage. Each stage is performed in a divide-and-conquer framework, as explained below.

\subsection{The model selection stage}

Given a data set with $n$ examples, we first split it randomly into $m$ different subsets, and distribute them onto $m$ machines.
Now we have $\frac{n}{m}$ samples on each machine. We will use $(\mathbf{X}^k, Y^k)$ to denote the subset on machine $k$.

At the model selection stage, we perform gLasso on each machine for a fixed number of $\lambda$ values. This can be done by any solver mentioned above. Then on the $k^{th}$ machine, we select the optimal model $\hat{M}_k$ through BIC criterion. Let $\hat{\beta}^k$ denote the (sparse) estimate on machine $k$. We define
\begin{eqnarray}
\hat{M}_k = \left\{i | \hat \beta^k_i \neq 0,~i = 1\cdots q \right\}, \nonumber
\end{eqnarray}
which denotes the set of groups that have non-zero weights in the local estimate on machine $k$.
Once the master node collects all the local estimates $\{ \hat{M}_k: k=1 \cdots m\}$, it combines the selected models using majority voting to get the final sparse model $\hat M$, that is,
\begin{eqnarray}\label{eq:MV}
\hat{M}=\left\{ i|\sum_{k=1}^m 1_{i\in \hat{M}_k}\geq \frac{m}{2} \right\},
\end{eqnarray}
where $1_{i \in \hat{M}_k}$ is an indicator function that has value $1$ if $i \in \hat{M}_k$ and $0$ otherwise. We can see that the final estimate $\hat{M}$ selects feature group $i$ if it is selected in no less than a half $(m/2)$ of the local models. Notice that, here, different from the original MESSAGE algorithm, the variables are selected in groups rather than in individuals.

\subsection{The coefficient estimation stage}

After we get the sparse pattern in the above step, we perform the coefficient estimation again in the divide-and-conquer setting, with the data split over $m$ machines. In this stage, we first distribute the selected model $\hat{M}$ to all the $m$ machines. Then we perform OLS on the local data subset sitting on each machine. Notice that we can now throw away the variables which are not in model $\hat{M}$, so $\beta_{\hat{M}}$ is estimated by:
\begin{eqnarray}\label{eq:localOLE}
\hat{\beta}_{\hat{M}}^k = \argmin_{\beta^k_{\hat{M}}} \left\| Y_{\hat{M}}^k - \mathbf{X}^k_{\hat{M}}\beta^k_{\hat{M}} \right\|_2^2,
\end{eqnarray}
where $\mathbf{X}^k_{\hat{M}}$ denotes the part of the $k^{th}$ design matrix (i.e., $\mathbf{X}^k$) that is selected by the sparse pattern $\hat{M}$, likewise for $Y^k_{\hat{M}}$.

After the master node collects all the local weights $\{ \hat{\beta}_{\hat{M}}^k: k=1 \cdots m \}$, the final weights are estimated by an average:
\begin{eqnarray}\label{eq:average}
\hat{\beta}_{\hat{M}} = \frac{1}{m} \sum_{k=1}^{m} \hat{\beta}^k_{\hat{M}}.
\end{eqnarray}
This step is the same as that in the MESSAGE algorithm.



\begin{algorithm}[t]
\caption{A Parallel Inference Algorithm for gLasso}
\label{alg-DC-gLasso}
\begin{algorithmic}[1]
\STATE Split the data randomly into $m$ subsets and distribute them on $m$ machines.
\STATE \textbf{On each machine:} do local group-Lasso $$\hat{\beta}^k = \argmin_{\beta}  \left\| Y^k - \mathbf{X}^k \beta \right\|_2^2 + \lambda \sum_{i=1}^{q} \left\| \beta_i \right\|_2.$$
\STATE Vote for the best sparse pattern $\hat M$ using the rule (\ref{eq:MV}).
\STATE \textbf{On each machine:} estimate the local linear regression weights $\hat{\beta}^k_{\hat M}$ by solving problem (\ref{eq:localOLE}).  
\STATE Average to obtain the final estimate $\hat{\beta}_{\hat M}$ via rule (\ref{eq:average}).
\end{algorithmic}
\end{algorithm}

\subsection{Overlapping Group-Lasso}


In real-world datasets, variables can often belong to more than one group, which leads to overlapping between groups. The overlapping group-Lasso~\cite{Jacob2009Group} is an extension of gLasso to deal with these overlaps. It exploits the gLasso penalty with duplicated variables to obtain a sparse solution whose support (i.e., nonzero components of the recovered coefficients) is a union of groups.

Formally, let $\nu_j\in\mathbb{R}^p$ denote a vector whose nonzero components are those positions corresponding to the features in group $j$, and let $\mathcal{V}_j\subseteq \mathbb{R}^p$ be the subspace of such possible vectors. Then, by duplicating the variables in the original design matrix $\bf X$, we obtain a new matrix $\mathbf{X}^{'} =(\mathbf{X}_1,...,\mathbf{X}_q)$, where $\mathbf{X}_j$ is a sub-matrix of $\bf X$ that corresponds to the $j^{th}$ feature group. Note that there could be overlaps between $\mathbf{X}_i$ and $\mathbf{X}_j$ if groups $i$ and $j$ overlap. The coefficient vector $\beta$ is given by $\beta=\sum_{j=1}^{q}\nu_j$, and the overlapping group-Lasso solves the optimization problem~\cite{hastie2015statistical}:
$$\hat{\beta}=\argmin_{\beta=\sum_{j=1}^{q}\nu_j,\nu_j\in\mathcal{V}_j}\Vert Y-\mathbf{X^{'}}(\sum_{j=1}^{q}\nu_j) \Vert_2^2+\lambda\sum_{j=1}^{q}\Vert \nu_j\Vert_2.$$
This problem can be solved using a gLasso solver after duplicating the overlapped variables~\cite{Jacob2009Group}, and can be solved directly by accelerated gradient descent~\cite{Lei2013Efficient}, proximal gradient method~\cite{argyriou2011efficient} and Alternating Direction Method of Multipliers (ADMM)~\cite{boyd2011distributed}.

We can extend DC-gLasso to perform parallel inference in the presence of group overlappings. The DC-gLasso with overlaps (DC-ogLasso) has the similar two-step procedure as DC-gLasso. Algorithm~\ref{alg-DC-ogLasso} outlines the procedure of DC-ogLasso. We can see that the key difference is in the model selection stage, where a distributed overlapping group-Lasso is carried out on each local machine using any solver as mentioned above, and instead of selecting features in groups, we adopt a {\it select-and-discard} strategy to select the sparse model. This strategy turns out to have a larger probability of selecting the correct model in practice compared with the select-in-groups strategy as in DC-gLasso.

The select-and-discard strategy works as follows. 
After the master node collects local estimates $\hat \beta^k$, a majority voting is applied to set each individual feature to get the sparse model $\hat M$.
In order to ensure that the selected features constitute a union of groups, which is consistent with the property of overlapping group-Lasso, we apply ``security check" to discard those features that are ``alone" in model $\hat M$. By saying a feature is ``alone", we mean that there does not exist a group that contains the feature and whose features are all selected in model $\hat M$. Therefore, the discarding process ensures that we obtain a model whose support is a union of groups.



\begin{algorithm}[t]
\caption{DC-gLasso with Overlap}
\label{alg-DC-ogLasso}
\begin{algorithmic}[1]
\STATE Split the data randomly into $m$ subsets and distribute them on $m$ machines.
\STATE \textbf{On each machine:} do local overlapping group-Lasso $$\hat{\beta}^k = \argmin_{\beta=\sum_{j=1}^{q}\nu_j,\nu_j\in\mathcal{V}_j}\Vert Y^k-\mathbf{X}^{'k}(\sum_{j=1}^{q}\nu_j) \Vert_2^2+\lambda\sum_{j=1}^{q}\Vert \nu_j\Vert_2.$$
\STATE Vote for the best sparse pattern $\hat M$ using the select-and-discard strategy (See text).
\STATE \textbf{On each machine:} estimate the local linear regression weights $\hat{\beta}^k_{\hat M}$ by solving problem (\ref{eq:localOLE}).  
\STATE Average to obtain the final estimate $\hat{\beta}_{\hat M}$ via rule (\ref{eq:average}).
\end{algorithmic}
\end{algorithm}

\section{Model selection consistency}


Wang et al.~\shortcite{wang2014median} provide a nice model selection consistency bound for MESSAGE, based on previous work on theoretical analysis of Lasso~\cite{Zhao2006On}. However, due to the complexity of various grouping structures of gLasso problems, it is very hard to derive a tight bound for gLasso, so is DC-gLasso.
In this section, we provide a theoretical result on model selection consistency of our algorithm. Let $\mathbf{M} = \{j | \beta_j \neq \mathbf{0}\}$ denote the sparsity pattern of the model parameter $\beta$. We make the same assumptions as ~\cite{bach2008consistency}:
\begin{enumerate}[(A)]
\item $\mathbf{X}$ and $Y$ have finite fourth order moments: $\mathbb{E}\left\| \mathbf{X} \right\|^4 < \infty  $ and $\mathbb{E}\left\| Y \right\|^4 < \infty$.
\item The joint matrix of second order moments $\mathbb{E} \mathbf{X}\mathbf{X}^\top \in \mathbb{R}^{p\times p}$ is invertible.
\item Strong condition for group-lasso consistency
$$\hspace{-.6cm}\max_{i\in \mathbf{M}^c}\frac{1}{d_i}\left\| (\mathbf{X}_i^\top \mathbf{X_M}) (\mathbf{X_M}^\top \mathbf{X_M})^{-1} \textrm{diag}\left( \frac{d_j}{ \left\| \mathbf{\beta}_j \right\| } \right)\mathbf{\beta_M} \right\| < 1,$$
\end{enumerate}
where $\mathbf{X_M}$ is the sub-matrix of $\mathbf{X}$ with the columns selected by $\mathbf{M}$, likewise for
$\mathbf{\beta_M}$ the sub-vector indexed by $\mathbf{M}$, and $\textrm{diag}(\frac{d_j}{ \left\| \mathbf{\beta}_j \right\| })$ is the block-diagonal matrix with $\frac{d_j}{\left\| \mathbf{\beta}_j\right\|} I_{p_j}$ on the diagonal. Then, we have the following result:
\begin{theorem}
\emph{(Model Selection Consistency)}
If each subset satisfies (A), (B) and (C), then there exist a sequence of $\lambda$ such that the sparsity pattern of the estimates given by this algorithm $M(\hat{\beta})$ converges in probability to $\mathbf{M}$ when $n/m$ goes to infinity.
\end{theorem}
\begin{proof}
We use $Pr_n$ to denote the probability of selecting the true model using group-Lasso with $n$ samples. Then we have
\begin{eqnarray}
Prob(\hat{M}=M) & \geq & Prob\left( \sum_{i=1}^{m} 1_{\hat{M}_i=M}\geq \frac{m}{2} \right) \nonumber \\
& \geq & 1 - \frac{Pr_{n/m}(1-Pr_{n/m})}{m(Pr_{n/m}-1/2)^2}, \nonumber
\end{eqnarray}
where the first inequality holds due to the selection rule of $\hat M$ in Eq.~(\ref{eq:MV}), and the second inequality holds due to Chebyshev's inequality.
Under those model assumptions, we can apply Theorem 2 in~\cite{bach2008consistency}, which states that there exist a sequence of $\lambda$ such that $Pr_n \rightarrow 1$ as $n \rightarrow \infty$. And $Prob(\hat{M}=M) \rightarrow 1$ follows consequently.
\end{proof}

\section{Experiments}

We present experimental results on both synthetic and real data sets to demonstrate the effectiveness of our algorithm.

\subsection{Synthetic datasets}

We first present results on synthetic data sets, which were commonly used to evaluate sparse learning methods.

\subsubsection{Performance}

We design a simulation model and evaluate our algorithm based on its performance on datasets drawn from the model. The model is similar to the ones introduced in \cite{yuan2006model} and \cite{yang2014fast}. The design matrix is generated as follows. The vectors $Z_i, i=1,2...q$ are generated from a multivariate normal distribution with a correlation matrix such that the correlation between $Z_i$ and $Z_j$ is $\rho$ for $i\neq j$. Let the $s$-sparse model's response vector $Y$ be
$$Y=\sum_{s|i,~1\leq i\leq q} \left( \frac{2}{3} Z_i-h_{i2} Z_i^2+\frac{1}{3}h_{i3}Z_i^3 \right) t_i + k \epsilon_0,$$
where $h_{i2},h_{i3}$ are positive real numbers such that $\Vert h_{i2} Z_i^2 \Vert_2=\Vert h_{i3} Z_i^3 \Vert_2=1$; $t_i=(-1)^{u_i}(3+v_i)$, with $u_i$ randomly drawn from set $\{0,1\}$ and $v_i$ drawn from $\mathcal{N}(0,1)$; $\epsilon_0$ is a noise vector drawn from $\mathcal{N}(0,1)$ and a scaling parameter $k$ ensures that the signal-to-noise ratio is $3.0$. The summation condition $s|i$ is true when $i$ is a multiple of $s$, which implies that only $1/s$ variables are involved in the model, so we call it a $s$-sparse model.

\begin{figure*}[htp] \vspace{-.5cm}  
\centering
\subfigure{
\includegraphics[width=0.31\textwidth,height=0.2\textheight]{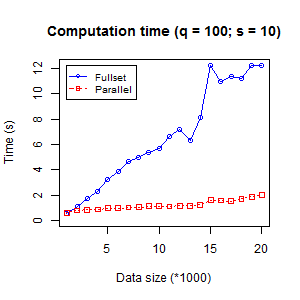}}\vspace{-.5cm}
\subfigure{
\includegraphics[width=0.31\textwidth,height=0.2\textheight]{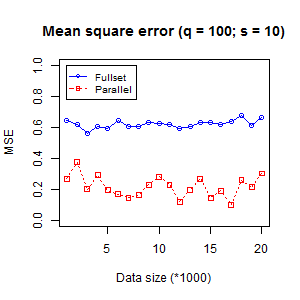}}\vspace{-.5cm}
\subfigure{
\includegraphics[width=0.31\textwidth,height=0.2\textheight]{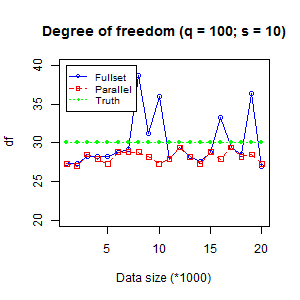}}\vspace{-.5cm}

\subfigure{
\includegraphics[width=0.31\textwidth,height=0.2\textheight]{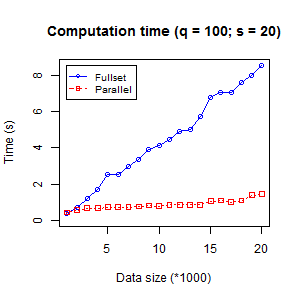}}\vspace{-.2cm}
\subfigure{
\includegraphics[width=0.31\textwidth,height=0.2\textheight]{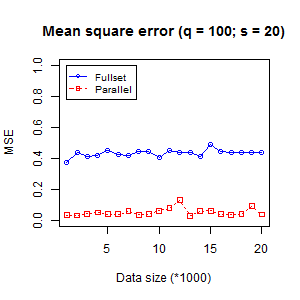}}\vspace{-.2cm}
\subfigure{
\includegraphics[width=0.31\textwidth,height=0.2\textheight]{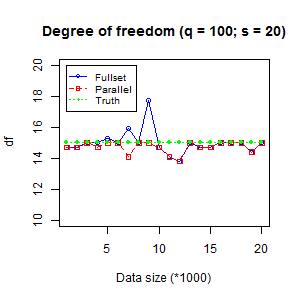}}\vspace{-.2cm}

\caption{Results in \textbf{Scenarios 1 \& 2} (group-Lasso part is implemented by R-package $SGL$).}\label{fig:sgl}\vspace{-.3cm}
\end{figure*}

When fitting it using a gLasso model, we treat the vectors $\{Z_i,h_{i2} Z_i^2,h_{i3}Z_i^3\}$ as a group. Compared with Eq. (\ref{eq2}), we have $\mathbf{X}_i=(Z_i,h_{i2} Z_i^2,h_{i3}Z_i^3)$ and $\beta_i=(\frac{2}{3}t_i,-t_i,\frac{1}{3}t_i)^\top$ when $i$ is a multiple of $s$, $\mathbf{0}$ otherwise. So $p=3q$, where $p$, $q$ and $n$ are the number of features, number of feature groups and sample size, respectively. The experiment task can be formulated as follows: given the response vector $Y$ of length $n$, the design matrix $\mathbf{X}$ of dimension $n\times p$ and the grouping information (i.e., three consecutive features in one group), recover the sparsity pattern and value of model coefficient $\beta$.

We consider the following two combinations of $q$, $s$ and $\rho$. In each setting, sample size $n$ varies from $1,000$ to $20,000$.
\begin{itemize}
\item[] \textbf{Scenario 1} $(p,q,s,\rho)=(300,100,10,0.5)$
\item[] \textbf{Scenario 2} $(p,q,s,\rho)=(300,100,20,0.5)$
\end{itemize}
In each scenario, we compare the performance of DC-gLasso with the counterpart algorithm on the full training set (denoted by gLasso).  
When performing DC-gLasso, we fix the subset size (i.e., the number of samples handled on one machine) at $1,000$, which means that as the sample size grows from $1,000$ to $20,000$, the number of machines we use increases from $1$ to $20$. The full-set gLasso problem and the subset gLasso problems in our algorithm are solved by the ``SGL" R-package, which implements the generalized gradient descent based algorithm~\cite{Noah2013A}. During each gLasso optimization, the best model is chosen using BIC from a solution path consisting twenty lambdas (the default setting of the ``SGL" package).

We compare the performance in terms of computation time, mean square error between the recovered $\hat{\beta}$ and the true $\beta$, and sparsity pattern of the recovered model (measured by the number of nonzero variables in $\hat{\beta}$). The results are shown in Fig.~\ref{fig:sgl} with ``Fullset", ``Parallel" and ``Truth" denoting the result of full-set inference, the result of parallel inference and the ground truth, respectively.
We can see that in both combinations of feature group number and sparsity, DC-gLasso obtains a huge reduction in computation time while doing a great job on recovering the sparsity pattern. With the inclusion of coefficient estimation stage, DC-gLasso is able to achieve a lower mean square error in estimating model coefficients. We can also see that DC-gLasso always selects fewer variables than the fullset inference method, but they achieve comparable results on recovering the model sparsity pattern.

To verify that our divide-and-conquer scheme can reduce computation time in grouped regression problems regardless of the gLasso solver, we repeat the experiment using another gLasso R-package ``gglasso", which implements groupwise-majorization-descent algorithm~\cite{yang2014fast}, in the following settings:
\begin{itemize}
\item[] \textbf{Scenario 3} $(p,q,s,\rho)=(3000, 1000, 10,0.5)$
\item[] \textbf{Scenario 4} $(p,q,s,\rho)=(3000, 1000, 20,0.5)$
\end{itemize}

Here the sample size $n$ still varies from $1,000$ to $20,000$ with the subset size fixed at $1,000$. During each gLasso optimization, the best model is chosen using BIC from a solution path consisting one hundred lambdas (the default setting of the ``gglasso" package). The results are shown in Fig.~\ref{fig:gglasso}. Still, our method offers nearly constant computing time when the sample size scales up while having lower mean square error and better recovering of sparsity pattern.

\subsubsection{Sample complexity measurement}

In this experiment, we compare the degree of freedom of coefficients recovered by fullset gLasso and DC-gLasso to measure the number of samples needed to recover the true sparsity pattern for each method in the following two scenarios:
\begin{itemize}
\item[] \textbf{Scenario 5} $(m,p,q,s,\rho)=(10,3000,1000,10,0)$
\item[] \textbf{Scenario 6} $(m,p,q,s,\rho)=(10,3000,1000,20,0)$
\end{itemize}

In both scenarios, $10$ machines are used and the correlation between two feature vectors is set to zero to simplify the setting, with the subset size varying from 150 to 575. Fig.~\ref{fig:cmplx} shows that DC-gLasso needs 325 subset samples in scenario $5$, where sparsity is $0.1$, and 200 subset samples in scenario $6$, where sparsity is 0.05, to recover the true sparsity. Meanwhile, fullset gLasso needs at most $1,500$ samples.

\begin{figure}[htbp]\vspace{-.2cm}
\centering

\subfigure{
\includegraphics[width=0.22\textwidth]{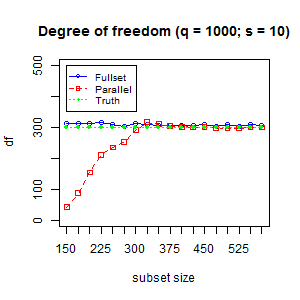}}\vspace{-.3cm}
\subfigure{
\includegraphics[width=0.22\textwidth]{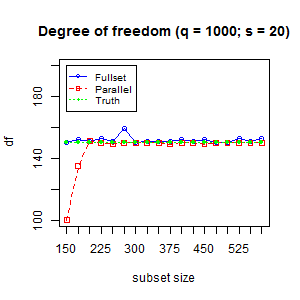}}\vspace{-.3cm}
\caption{Sample complexity results in \textbf{Scenarios 5 \& 6} (group-Lasso part is implemented by R-package $gglasso$).}
\label{fig:cmplx}\vspace{-.2cm}
\end{figure}

\subsubsection{Performance with overlapping groups}

To evaluate the performance of DC-ogLasso, we use a similar method of generating synthetic data as in~\cite{Lei2013Efficient}. Consider the linear regression problem, where the nonzero components of the coefficient form a union of groups. The group indices are predefined such that $G_1=\{1,2,...,10\},G_2=\{6,7,...,15\},G_3=\{11,12,...,20\},...,$, with each group having exactly $10$ features and overlapping half of the previous group. Each feature group is selected in probability $0.1$. The design matrix $\bf X$ and selected components of the coefficient $\beta$ are all generated from the multivariate $\mathcal N(0,1)$ without correlation. The standard gaussian noise $\epsilon$ is generated with a scaling factor $0.01$. Then the response vector $Y$ is computed as in Eq.~(\ref{eq1}).

\begin{figure}[htbp]\vspace{-.3cm}
\centering

\subfigure{
\includegraphics[width=0.20\textwidth]{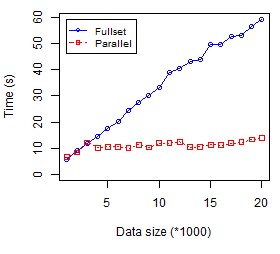}}\vspace{-.3cm}
\subfigure{
\includegraphics[width=0.20\textwidth]{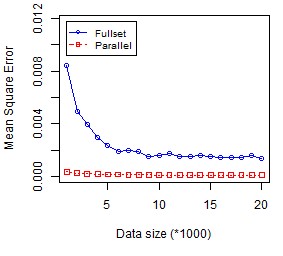}}\vspace{-.3cm}
\caption{Performance of DC-ogLasso (overlapping group-Lasso part is implemented by R-package $grpregOverlap$)}
\label{fig:overlap}
\end{figure}

The feature size is fixed $(p=1,000)$ and sample size $n$ varies from $1,000$ to $20,000$. For DC-ogLasso, each machine handles $1,000$ samples. The result is shown in Fig.~\ref{fig:overlap}. Still, our DC method greatly shortens the computation time when $n$ gets large.
In terms of model selection consistency, when $n\geq 2,000$, DC-ogLasso using select-and-discard strategy has no less than $89$ times out of $100$ when it selects the correct model, while this number for overlapping group-lasso is $91$. When using select-in-groups strategy, it drops to $87$. In terms of the mean square error of the recovered model coefficient, DC-ogLasso outperforms overlapping group-lasso, thanks to the inclusion of the coefficient estimation stage.

\begin{figure*}[htp]\vspace{-.5cm}
\centering

\subfigure{
\includegraphics[width=0.31\textwidth,height=0.22\textheight]{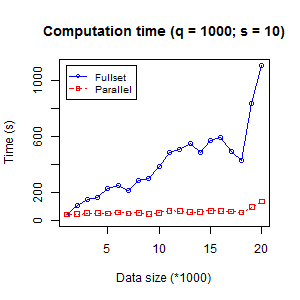}}\vspace{-.4cm}
\subfigure{
\includegraphics[width=0.31\textwidth,height=0.22\textheight]{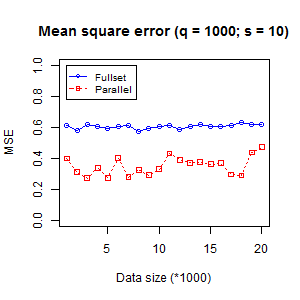}}\vspace{-.4cm}
\subfigure{
\includegraphics[width=0.31\textwidth,height=0.22\textheight]{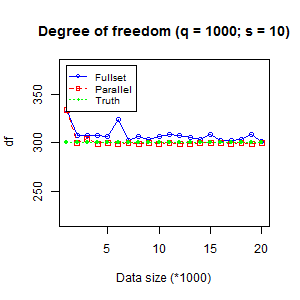}}\vspace{-.4cm}
\subfigure{
\includegraphics[width=0.31\textwidth,height=0.22\textheight]{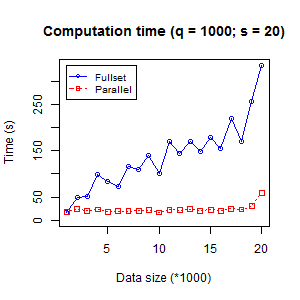}}\vspace{-.2cm}
\subfigure{
\includegraphics[width=0.31\textwidth,height=0.22\textheight]{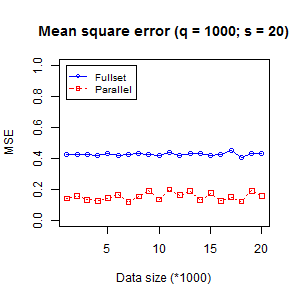}}\vspace{-.2cm}
\subfigure{
\includegraphics[width=0.31\textwidth,height=0.22\textheight]{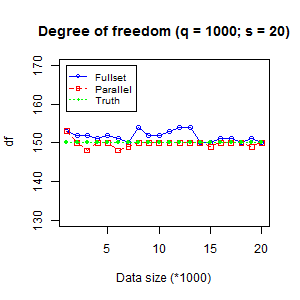}}\vspace{-.2cm}
\caption{Results in \textbf{Scenarios 3 \& 4} (group-Lasso part is implemented by R-package $gglasso$)}
\label{fig:gglasso}\vspace{-.2cm}
\end{figure*}

\subsection{MEMset Donor dataset}

Groupwise inference methods like gLasso can be applied to predict donor splice sites, which play a key role in finding genes. The MEMset Donor dataset\footnote{Available at http://genes.mit.edu/burgelab/maxent/ssdata/} is one of the splice sites datasets. It consists of a training set of $8,415$ true and $179,438$ false human donor sites with an additional test set of $4,208$ true and $89,717$ false donor sites. An instance of donor site sample is a sequence of $7$ factors with four levels ${A,C,G,T}$ (Please see Yeo and Burge~\shortcite{yeo2004maximum} for details). Here we follow an approach in~\cite{meier2008group}, which separates the original training set into an balanced training set of 5,610 true and 5,610 false donor sites, and a unbalanced validation set of 2,805 true and 59,804 false donor sites.

The data are represented as a collection of all factor interactions up to degree $2$. Each interaction is encoded using dummy variables and treated as a group, leading to $63$ groups of size varying from $4$ to $4^3$, a total 2,604-dimension feature space. We train the model on the balanced training set, then choose the best threshold parameter $\tau$ for classifying output over the trained model. That is, we assign sample $i$ to the true class if $p(x_i~\in~true~donor~sites)>\tau$ and to the false class otherwise. Finally, we evaluate the model using Pearson correlation between true class labels and predicted class labels on test set.

We compare the performance of logistic gLasso~\cite{meier2008group} and logistic DC-gLasso on this dataset. Logistic DC-gLasso is a modified version of DC-gLasso on logistic regression. In logistic DC-gLasso, the square loss is replaced by the cross-entropy loss and the distributed least square regression in the coefficient estimation stage replaced by the distributed logistic regression.

\begin{figure}[htbp]\vspace{-.4cm}

\centering

\subfigure{\hspace{-.3cm}
\includegraphics[width=0.17\textwidth,height=0.17\textheight]{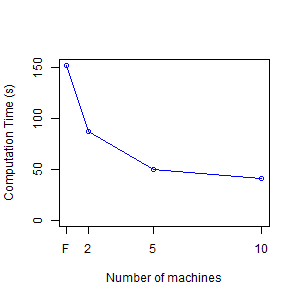}}\vspace{-.2cm}\hspace{-.4cm}
\subfigure{
\includegraphics[width=0.17\textwidth,height=0.17\textheight]{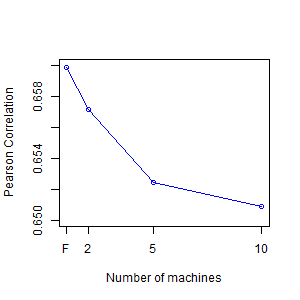}}\vspace{-.2cm}\hspace{-.4cm}
\subfigure{
\includegraphics[width=0.17\textwidth,height=0.17\textheight]{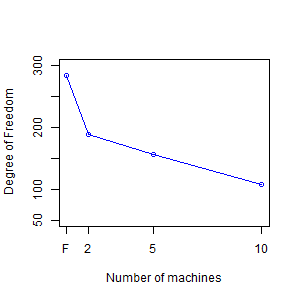}}\vspace{-.0cm}\hspace{-.10cm}
\caption{Results on the MEMset Donor dataset (group-Lasso part is implemented by R-package $gglasso$)}
\label{fig:splice}
\end{figure}

Fig.~\ref{fig:splice} shows the performance of fullset logistic gLasso (denoted by "F") and logistic DC-gLasso using 2, 5 and 10 machines on the MEMset Donor dataset. In our experiment, the correlation results are 0.6598 for fullset inference, 0.6571, 0.6524 and 0.6508 for logistic DC-gLasso on 2, 5 and 10 machines. Also note that the best reported result using logistic gLasso over degree 2 interactions is 0.6593~\cite{meier2008group}. Although our DC method suffers a small correlation loss compared with the original fullset approach, it enjoys a large reduce in training time: when running logistic DC-gLasso on 2 machines, only 87 seconds are needed to complete the original 151-second-job; when we increase the machine quantity to 5, it reduces to 50 seconds. In addition, the "Degree of Freedom" graph shows that the model we obtain shrinks as we split the training data in more parts, which makes sense because it's more difficult for more than half of all subsets to "vote" for one particular feature as the sample size in one subset declines, thus leading to a sparser model.


\section{Conclusions}

We propose DC-gLasso, a parallel inference method for group-Lasso and demonstrate its excellent performance on large data sets. DC-gLasso enjoys the advantage of parallel computing while minimizing the communication cost between machines. It successfully provides a scalable solution to the group-Lasso method: with enough machines, DC-gLasso can complete the task in a roughly constant short time regardless of the data size, with low mean square error and high probability of selecting the correct model.

\bibliographystyle{named}
\bibliography{ijcai16}

\end{document}